\documentclass{article}

\usepackage{times}
\usepackage{graphicx} 
\usepackage{subfigure}

\usepackage{amsthm}

\newtheorem{defi}{Definition}
\newtheorem{lemm}{Lemma}
\newtheorem{thm}{Theorem}
\newtheorem{prop}{Proposition}

\newtheorem{cor}{Corollary}
\usepackage{natbib,amsmath,amssymb}

\usepackage{algorithm}
\usepackage{algorithmic}

\usepackage{hyperref}

\usepackage[accepted]{icml2017}

\icmltitlerunning{Algorithmic Stability and Hypothesis Complexity}

\begin{document}

\twocolumn[
\icmltitle{Algorithmic Stability and Hypothesis Complexity}




\icmlsetsymbol{equal}{*}

\begin{icmlauthorlist}
\icmlauthor{Tongliang Liu}{to}
\icmlauthor{G{\'a}bor Lugosi}{goo,ed1,ed2}
\icmlauthor{Gergely Neu}{ed3}
\icmlauthor{Dacheng Tao}{to}
\end{icmlauthorlist}

\icmlaffiliation{to}{UBTech Sydney AI Institute, School of IT, FEIT, The University of Sydney, Australia}
\icmlaffiliation{goo}{Department of Economics and Business, Pompeu Fabra University, Barcelona, Spain}
\icmlaffiliation{ed1}{ICREA, Pg.~Llus Companys 23, 08010 Barcelona, Spain}
\icmlaffiliation{ed2}{Barcelona Graduate School of Economics}
\icmlaffiliation{ed3}{AI group, DTIC, Universitat Pompeu Fabra, Barcelona, Spain}

\icmlcorrespondingauthor{Tongliang Liu}{tliang.liu@gmail.com}
\icmlcorrespondingauthor{G{\'a}bor Lugosi}{gabor.lugosi@upf.edu}
\icmlcorrespondingauthor{Gergely Neu}{gergely.neu@gmail.com}
\icmlcorrespondingauthor{Dacheng Tao}{dacheng.tao@sydney.edu.au}




\icmlkeywords{boring formatting information, machine learning, ICML}

\vskip 0.3in
]



\printAffiliationsAndNotice{}  

\begin{abstract}
We introduce a notion of algorithmic stability of learning algorithms---that we term \emph{argument stability}---that
captures stability of the hypothesis output by the learning algorithm in the normed space of functions from which
hypotheses are selected. The main result of the paper bounds the generalization error of any learning algorithm
in terms of its argument stability. The bounds are based on martingale inequalities in the Banach space to
which the hypotheses belong.
We apply the general bounds to bound the performance of some learning algorithms based on
empirical risk minimization and stochastic gradient descent.
\end{abstract}

\section{Introduction}

Many efforts have been made to analyze various notions of algorithmic stability and prove that a broad spectrum of
learning algorithms are stable in some sense. Intuitively, a learning algorithm is said to be stable if slight perturbations in
the training data result in small changes in the output of the algorithm, and these changes vanish as the data set
grows bigger and bigger \citep{bonnans2013perturbation}. For example, \citet{devroye1979distribution}, \citet{lugosi1994posterior}, and
\citet{zhang2003leave} showed that several non-parametric learning algorithms are stable;
\citet{bousquet2002stability} proved that $\ell_2$ regularized learning algorithms are uniformly stable;
\citet{poggio2009sufficient} generalized Bousquet and Elisseeff's results and proved that regularized learning algorithms with strongly
convex penalty functions on bounded domains, e.g., $\ell_p$ regularized learning algorithms for $1<p\leq2$, are also uniformly stable;
\citet{hardt2015train} showed that parametric models trained by stochastic gradient descent algorithms are uniformly stable; and
\citet{liu2016algorithm} proved that tasks in multi-task learning can act as regularizers and that multi-task learning in a very general
setting will therefore be uniformly stable under mild assumptions.

The notion of algorithmic stability has been an important tool in deriving theoretical guarantees of the generalization
abilities of learning algorithms.
Various notions of stability have been introduced and have been exploited to derive generalization bounds.
For some examples, \citet{mukherjee2006learning} proved that a statistical form of leave-one-out stability is a sufficient and necessary
condition for the generalization and learnability of empirical risk minimization learning algorithms;
\citet{shalev2010learnability} defined a weaker notion, the so-called ``on-average-replace-one-example stability'', and showed
that this condition is both sufficient and necessary for the generalization and learnability of a general learning setting.

In this paper we study learning algorithms that select a hypothesis (i.e., a function used for prediction) from a
certain fixed class of functions belonging to a separable Banach space. We introduce a notion of \emph{argument stability}
which measures the impact of changing a single training example on the hypothesis selected by the learning algorithm.
This notion of stability is stronger than uniform algorithmic stability of \citet{bousquet2002stability} that is only concerned about
the change in the loss but not the hypothesis itself. However, as we will show, the new notion is still quite natural and holds for a
variety of learning algorithms. On the other hand, it allows one to exploit martingale inequalities \cite{boucheron2013concentration} in
the Banach space of the hypotheses. Indeed, the performance bounds we derive for stable algorithms depend on
characteristics related to the \emph{martingale type} of the Banach space.

Generalization bounds typically depend on the complexity of a class of hypotheses that can be chosen by the learning algorithm.
Exploiting the local estimates of the complexity of the predefined hypothesis class is a promising way to obtain sharp bounds. Building on
martingale inequalities in the Banach space of the hypotheses, we define a subset of the predefined hypothesis class, whose elements will
(or will have a high probability to) be output by a learning algorithm, as the \emph{algorithmic hypothesis class}, and study the
complexity of the algorithmic hypothesis class of argument-stable learning algorithms. We show that, if the hypotheses belong to a
Hilbert space, the upper bound of the Rademacher complexity of the algorithmic hypothesis class will converge at a fast rate of order
$O(1/n)$, where $n$ is the sample size.

The rest of the paper is organized as follows.
Section \ref{stability} introduces the mathematical framework and the proposed notion
of algorithmic stability.
Section \ref{mainsection} presents the main results of this study, namely the generalization bounds
in terms of argument stability.
Section \ref{application} specializes the results to some learning algorithms, including empirical risk minimization and stochastic gradient descent.
Section \ref{conclusion} concludes the paper.

\section{Algorithmic Stability and Hypothesis Class}\label{stability}
We consider the classical statistical learning problem, where the value of a real random variable $Y$ is to be predicted based on the
observation of an another random variable $X$. Let $S$ be a training sample of $n$ i.i.d.\ pairs of random variables
$Z_1=(X_1,Y_1),\ldots,Z_n=(X_n,Y_n)$ drawn from a fixed distribution $P$ on a set $\mathcal{Z}=\mathcal{X}\times \mathcal{Y}$, where
$\mathcal{X}$ is the so-called \emph{feature space}. A learning algorithm $\mathcal{A}: S\in\mathcal{Z}^n\mapsto h_S\in H$ is a mapping
from $\mathcal{Z}^n$ to a hypothesis class $H$ that we assume to be a subset of a separable Banach space $(\mathfrak{B},\|\cdot\|)$.
We focus on \emph{linear} prediction problems, that is, when $h(x)$ is a linear functional of $x$.
We write $h(x)= \langle h, x \rangle$.
In other words, we assume that the feature space $\mathcal{X}$ is the algebraic dual of the Banach space $\mathfrak{B}$.
We denote the norm in $\mathcal{X}$ by $\|\cdot\|_*$.
The output $h_S$ of the learning algorithm is a hypothesis used for predicting the value for $Y$.

An important special case is when $\mathfrak{B}$ is a Hilbert space. In that case we may assume that $\mathcal{X}=\mathfrak{B}$
and that $\langle h, x \rangle$ is the inner product in $\mathfrak{B}$.

The quality of the predictions made by any hypothesis will be measured by a loss function
$\ell:\mathfrak{B}\times\mathcal{Z}\rightarrow \mathbb{R}_+$ (where $\mathbb{R}_+$ denotes the set of positive reals).
Specifically, $\ell(h,Z)$ measures the loss of predicting an example $Z$ using a hypothesis $h$.

The \emph{risk} of $h\in H$ is defined by
\begin{eqnarray*}
R(h)=\mathbb{E}\ell(h,Z)~;
\end{eqnarray*}
while the \emph{empirical risk} is
\begin{eqnarray*}
R_S(h)=\frac{1}{n}\sum_{i=1}^{n}\ell(h,Z_i)~.
\end{eqnarray*}

For the output $h_S$ of a learning algorithm $\mathcal{A}$, the generalization error is defined as
\begin{eqnarray}\label{gerror1}
R(h_S)-R_S(h_S)~.
\end{eqnarray}

The notion of \emph{algorithmic stability} was proposed to measure the changes of outputs of a learning algorithm when the input is
changed.
Various ways have been introduced to measure algorithmic stability.
Here we recall the notion of \emph{uniform stability} defined by \citet{bousquet2002stability} for comparison purposes.
This notion of stability relies on the altered sample $S^i=\{Z_1,\ldots,Z_{i-1},Z'_i,Z_{i+1},\ldots,Z_n\}$, the sample
$S$ with the $i$-th example being replaced by an independent copy of $Z_i$.

\begin{defi}[Uniform Stability]
A learning algorithm $\mathcal{A}$ is $\beta(n)$-uniformly stable with respect to the loss function $\ell$ if
for all $i\in\{1,\ldots,n\}$,
\begin{eqnarray*}
|\ell(h_S,Z)-\ell(h_{S^i},Z)|\leq \beta(n)~,
\end{eqnarray*}
with probability one, where $\beta(n)\in\mathbb{R}_+$~.
\end{defi}
We propose the following, similar, notion that ``acts'' on the hypotheses directly, as opposed to
the losses.
\begin{defi}[Uniform Argument Stability]
A learning algorithm $\mathcal{A}$ is $\alpha(n)$-uniformly argument stable if
for all $i\in\{1,\ldots,n\}$,
\begin{eqnarray*}
\|h_S-h_{S^i}\|\leq \alpha(n)~,
\end{eqnarray*}
with probability one, where $\alpha(n)\in\mathbb{R}_+$~.
\end{defi}

The two notions of stability are closely related: Intuitively, if the loss $\ell(h,z)$ is a sufficiently smooth function of $h$, then
uniform argument stability should imply uniform stability. To make this intuition precise, we define the notion of
\emph{Lipschitz-continuous} loss functions below.
\begin{defi}[$L$-Lipschitz Loss Function]\label{def:lipschitz}
 The loss function $\ell:\mathfrak{B}\times\mathcal{Z}\rightarrow \mathbb{R}_+$ is $L$-Lipschitz for an $L>0$ if
 \[
 \left|\ell(h,z)-\ell(h',z) \right| \le L \left|\langle h, x\rangle-\langle h',x\rangle\right|
\]
holds for all $z\in \mathcal{Z}$ and $h,h'\in H$.
\end{defi}
Additionally assuming that $\|X\|_*$ is bounded by some $B>0$
with probability one, it is easy to see that an $\alpha(n)$-uniformly argument stable learning algorithm is uniformly stable with
$\beta(n)= LB\alpha(n)$, since
\[
\|h_S-h_{S^i}\| = \sup_{x\in \mathcal{X}: \|x\|_*\le 1} \left(\langle h_S, x\rangle- \langle h_{S^i}, x\rangle\right).
\]
 However, the reverse implication need not necessarily hold and hence
uniform argument stability is a stronger notion.

In the rest of the paper, we will focus on $L$-Lipschitz loss functions and assume that $\|X\|_*\le B$ holds almost surely. These
assumptions are arguably stronger than those made by \citet{bousquet2002stability} who only require that the loss function be bounded. In
contrast, our results will require that the loss $\ell(h,z)$ be Lipschitz in the linear form $\langle h, x\rangle$, which is only slightly
more general than assuming generalized linear loss functions. Nevertheless, these stronger assumptions will enable us to prove stronger
generalization bounds.

The relationship between argument stability and generalization performance hinges
on a property of the Banach space $\mathfrak{B}$ that is closely related to
the \emph{martingale type} of the space---see
\citet{Pis11} for a comprehensive account. For concreteness
we assume that the Banach space $\mathfrak{B}$
is $(2,D)$-smooth (or of martingale type 2) for
some $D>0$. This means that for all $h,h'\in\mathfrak{B}$,
\begin{eqnarray*}
\|h+h'\|^2+\|h-h'\|^2\leq 2\|h\|^2+2D^2\|h'\|^2~.
\end{eqnarray*}
Note that Hilbert spaces are $(2,1)$-smooth.
The property we need is described in the following result
of \citep{pinelis1994optimum}:

\begin{prop}
\label{thm:pinelis}
  Let $D_1,\ldots,D_n$ be a martingale difference sequence taking values in a separable $(2,D)$-smooth Banach space
$\mathfrak{B}$. Then for any $\epsilon>0$,
\begin{eqnarray*}
\mathbb{P}\left(\sup_{n\geq 1}\left\|\sum_{t=1}^{n}D_t\right\|\geq c \epsilon\right)\leq 2\exp\left(-\frac{\epsilon^2}{2D^2}\right)~,
\end{eqnarray*}
where $c$ is a constant satisfying that $\sum_{t=1}^{\infty}\|D_t\|_\infty^2\leq c^2$
(and $\|D_t\|_\infty$ is the essential supremum of the random variable $\|D_t\|$).
\end{prop}

Our arguments extend, in a straightforward manner, to more general Banach spaces
whenever exponential tail inequalities for bounded martingale sequences
similar to Proposition \ref{thm:pinelis} are available. We stay with the assumption
of $(2,D)$-smoothness for convenience and because it applies to the
perhaps most important special case when $\mathfrak{B}$ is a Hilbert space.
We refer to \citet{RaSr15} for more information of martingale inequalities
of this kind.

A key property of stable algorithms, implied by the martingale inequality,
is that the hypothesis $h_S$ output by the algorithm is concentrated---in the Banach space $\mathfrak{B}$---around its expectation
$\mathbb{E} h_S$. This is established in the next simple lemma.

\begin{lemm}\label{lemma1}
Let the Banach space $\mathfrak{B}$ be $(2,D)$-smooth. If a learning algorithm $\mathcal{A}$ is $\alpha(n)$-uniformly argument stable, then, for any $\delta>0$,
\begin{eqnarray*}
\mathbb{P}\left(\left\|h_S-\mathbb{E}h_S\right\|\leq D\alpha(n)\sqrt{2n\log(2/\delta)}\right)\geq 1-\delta~.
\end{eqnarray*}
\end{lemm}

\noindent
\begin{proof}
Introduce the martingale differences
\begin{eqnarray*}
D_t=\mathbb{E}(h_S|Z_1,\ldots,Z_t)-\mathbb{E}(h_S|Z_1,\ldots,Z_{t-1})
\end{eqnarray*}
so that
\begin{eqnarray*}
h_S- \mathbb{E}h_S=\sum_{t=1}^{n}D_t~.
\end{eqnarray*}
We have
\begin{eqnarray*}
  \lefteqn{
    \sum_{t=1}^{\infty}\|D_t\|_\infty^2  } \\
&= &\sum_{t=1}^{n}\|\mathbb{E}(h_S|Z_1,\ldots,Z_t)-\mathbb{E}(h_S|Z_1,\ldots,Z_{t-1})\|_\infty^2\\
&= &\sum_{t=1}^{n}\|\mathbb{E}(h_S-h_{S^t}|Z_1,\ldots,Z_t)\|_\infty^2 \\
&\leq &\sum_{t=1}^{n}(\mathbb{E}(\|(h_S-h_{S^t}\|_\infty|Z_1,\ldots,Z_t))^2 \\
&\leq & n\alpha(n)^2~.
\end{eqnarray*}
Thus, by Proposition \ref{thm:pinelis},
we have
\begin{eqnarray*}
\mathbb{P}\left(\left\|h_S-E_Sh_S\right\|\geq \alpha(n)D\sqrt{2n\log(2/\delta)}\right)\leq \delta
\end{eqnarray*}
for  $\delta= 2\exp\left(-\frac{\epsilon^2}{2D^2}\right)$.
\end{proof}


\section{Algorithmic Rademacher Complexity and Generalization Bound}\label{mainsection}

The concentration result of Lemma \ref{lemma1} justifies the following definition of
the ``algorithmic hypothesis class'': since with high probability $h_S$ concentrates
around its expectation $\mathbb{E}h_S$, what matters in the generalization performance of
the algorithm is the complexity of the ball centered at $\mathbb{E}h_S$ and
\emph{not that of the entire hypothesis class} $H$. This observation may lead to
significantly improved performance guarantees.

\begin{defi}[Algorithmic Hypothesis Class]
  For a sample size $n$ and confidence parameter $\delta>0$,
let $r=r(n,\delta)=D\alpha(n)\sqrt{2n\log(2/\delta)}$ and define
  the algorithmic hypothesis class of a stable learning algorithm by
\begin{eqnarray*}
B_r=\{h\in H|\left\|h-\mathbb{E}h_S\right\|\leq r(n,\delta)\}~.
\end{eqnarray*}
\end{defi}

Note that, by Lemma \ref{lemma1}, $h_S\in B_r$ with
probability at least $1-\delta$.

We bound the generalization error (\ref{gerror1}) in terms of the
Rademacher complexity \cite{bartlett2003rademacher} of the algorithmic hypothesis class.
The Rademacher complexity of a hypothesis class $H$ on the feature space $\mathcal{X}$ is defined
as
\begin{eqnarray*}
\mathfrak{R}(H)=\mathbb{E}\sup_{h\in H}\frac{1}{n}\sum_{i=1}^{n}\sigma_i \langle h, X_i \rangle~,
\end{eqnarray*}
where $\sigma_1,\ldots,\sigma_n$ are i.i.d.\ Rademacher variables that are uniformly distributed in $\{-1,+1\}$.

The next theorem shows how the Rademacher complexity of the algorithmic hypothesis
class can be bounded. The bound depends on the \emph{type} of the feature space
$\mathcal{X}$. Recall that the Banach space $(\mathcal{X},\|\cdot\|_*)$ is of
type $p\ge 1$ if there exists a constant $C_p$ such that for all $x_1,\ldots,x_n\in \mathcal{X}$,
\[
 \mathbb{E} \left\|\sum_{i=1}^n \sigma_i x_i \right\|_* \le C_p \left(\sum_{i=1}^n \|x_i\|_*^p\right)^{1/p}~.
 \]
 In the important special case when $\mathcal{X}$ is a Hilbert space, the
 space is of type $2$ with constant $C_2=1$.

\begin{thm}\label{thmone}
  Assume that $\mathfrak{B}$ is a $(2,D)$-smooth Banach space and that its
  dual $\mathcal{X}$ is of type $p$. Suppose that the marginal distribution of the $X_i$
  is such that $\|X_i\|_*\le B$ with probability one, for some $B>0$. If
  a learning algorithm is $\alpha(n)$-uniformly argument stable, then
the Rademacher complexity of the algorithmic hypothesis class $B_r$ on the feature space satisfies
\begin{eqnarray*}
\mathfrak{R}(B_{r})\leq DC_pB \sqrt{2\log(2/\delta)} \alpha(n)n^{-1/2+1/p}~.
\end{eqnarray*}
In particular, when $\mathfrak{B}$ is a Hilbert space, the bound simplifies to
\[
\mathfrak{R}(B_{r})\leq B \sqrt{2\log(2/\delta)}\alpha(n)~.
\]
\end{thm}

\noindent
\begin{proof}
We have
\begin{eqnarray*}
 \lefteqn{ \mathfrak{R}( B_r) } \\
 & = & \mathbb{E}\sup_{h\in B_r}\frac{1}{n}\sum_{i=1}^{n}\sigma_i\langle h, X_i \rangle\\
&=& \mathbb{E}\sup_{h\in B_{r}}\frac{1}{n}\sum_{i=1}^{n}(\sigma_i \langle h, X_i \rangle\\
&&-\sigma_i\mathbb{E} \langle h_S, X_i \rangle+\sigma_i\mathbb{E}\langle h_S, X_i \rangle)\\
&=& \mathbb{E}\sup_{h\in B_{r}}\frac{1}{n}\sum_{i=1}^{n}\sigma_i(\langle h, X_i \rangle-\mathbb{E}\langle h_S, X_i \rangle)\\
&=&\mathbb{E}\sup_{h\in B_{r}}\frac{1}{n}\sum_{i=1}^{n}\sigma_i\left<h-\mathbb{E}h_S,X_i\right>\\
&\leq& \mathbb{E}\sup_{h\in B_{r}}\frac{1}{n}\|h-\mathbb{E}h_S\|\left\|\sum_{i=1}^{n}\sigma_iX_i\right\|_*~\\
&\leq& \frac{r}{n} \mathbb{E}\left\|\sum_{i=1}^{n}\sigma_iX_i\right\|_*\\
&\leq& \frac{1}{n} \alpha(n)D\sqrt{2n\log(2/\delta)} C_p \left(\sum_{i=1}^{n}\|X_i\|_*^p\right)^{1/p}\\
&\leq & DC_pB \sqrt{2\log(2/\delta)} \alpha(n) n^{-1/2+1/p}~,
\end{eqnarray*}
concluding the proof.
\end{proof}

The theorem above may be easily used to bound the performance of an
$\alpha(n)$-uniformly argument stable learning algorithm. For simplicity,
we state the result for Hilbert spaces only. The extension to $(2,D)$-smooth
Banach spaces with a type-$p$ dual is straightforward.

\begin{cor}
\label{cor1}
  Assume that $\mathfrak{B}$ is a separable Hilbert space.
  Suppose that the marginal distribution of the $X_i$
  is such that $\|X_i\|_*\le B$ with probability one, for some $B>0$
  and that the
  loss function is bounded and Lipschitz, that is,
 $\ell(h,Z)\le M$ with probability one for some $M>0$ and
  $\left|\ell(h,z)-\ell(h',z) \right| \le L \left|\langle h, x\rangle-\langle h', x\rangle\right|$ for all $z\in \mathcal{Z}$ and $h,h'\in H$.
  If
  a learning algorithm is $\alpha(n)$-uniformly argument stable, then
  its generalization error is bounded as follows. With probability at least
  $1-2\delta$,
\begin{eqnarray*}
  \lefteqn{  R(h_S)-R_S(h_S) } \\
  &\le &  2LB \sqrt{2\log(2/\delta)}\alpha(n)+ M\sqrt{\frac{\log(1/\delta)}{2n}}~.
\end{eqnarray*}
\end{cor}

\noindent
\begin{proof}
Note first that, by Lemma \ref{lemma1}, with probability at least $1-\delta$,
\begin{eqnarray*}
  R(h_S)-R_S(h_S) \le \sup_{h\in B_r} (R(h) - R_S(h))~.
\end{eqnarray*}
On the other hand, by the boundedness of the
loss function, and the bounded differences inequality, with probability
at least $1-\delta$,
\begin{eqnarray*}
\lefteqn{  \sup_{h\in B_r}\left(R(h)-R_S(h)\right)  } \\
  & \leq &
\mathbb{E} \sup_{h\in B_r}\left(R(h)-R_S(h)\right) + M\sqrt{\frac{\log(1/\delta)}{2n}} \\
& \le &
  2\mathfrak{R}(\ell\circ B_r)+M\sqrt{\frac{\log(1/\delta)}{2n}}~,
\end{eqnarray*}
where  $\ell\circ H$ denotes the set of compositions of functions $\ell$ and $h\in H$.
By the Lipschitz property of the loss function and a standard contraction
argument, i.e., Talagrand Contraction Lemma \citep{ledoux2013probability}, we have,
\begin{eqnarray*}
 \mathfrak{R}(\ell\circ B_r) &\leq & L\cdot \mathfrak{R}(B_r)
\\
&\leq& LB \sqrt{2\log(2/\delta)}\alpha(n)~.
\end{eqnarray*}
\end{proof}

Note that the order of magnitude of $\alpha(n)$ of many stable
algorithms is of order $O(1/n)$. For the notion of uniform stability, such
bounds appear in
\citet{lugosi1994posterior,bousquet2002stability,poggio2009sufficient,hardt2015train,liu2016algorithm}.
As we will show in the examples below, many of these learning algorithms even have uniform argument stability
of order $O(1/n)$.
In such cases the bound of Corollary \ref{cor1} is essentially
equivalent of the earlier results cited above. The bound
is dominated by
the term $M\sqrt{\frac{\log(1/\delta)}{2n}}$ present by using the
bounded differences inequality. Fluctuations of the order of
$O(n^{-1/2})$ are often inevitable, especially when $R(h_S)$ is not
typically small. When small risk is reasonable to expect,
one may use
more advanced concentration inequalities with second-moment information, at the
price of replacing the generalization error by the so-called
``deformed'' generalization error $R(h_S)-\frac{a}{a-1}R_S(h_S)$ where
$a>1$. The next theorem derives such a bound, relying on techniques
developed by \citet{bartlett2005local}. This result improves
essentially on earlier stability-based bounds.

\begin{thm}\label{thmtwo}
  Assume that $\mathfrak{B}$ is a separable Hilbert space.
  Suppose that the marginal distribution of the $X_i$
  is such that $\|X_i\|_*\le B$ with probability one, for some $B>0$
  and that the
  loss function is bounded and Lipschitz, that is,
 $\ell(h,Z)\le M$ with probability one for some $M>0$ and
  $\left|\ell(h,z)-\ell(h',z) \right| \le L \left|\langle h, x\rangle-\langle h', x\rangle\right|$ for all $z\in \mathcal{Z}$ and $h,h'\in H$.
Let $a>1$.
  If
  a learning algorithm is $\alpha(n)$-uniformly argument stable, then,
  with probability at least
  $1-2\delta$,
\begin{eqnarray*}
\lefteqn{R(h_S)-\frac{a}{a-1}R_S(h_S)} \\
&&\leq 8LB \sqrt{2\log(2/\delta)}\alpha(n)+\frac{(6a+8)M\log(1/\delta)}{3n}~.
\end{eqnarray*}
\end{thm}

The proof of Theorem \ref{thmtwo} relies on techniques developed by
\citet{bartlett2005local}.
In particular, we make use of the following result.

\begin{prop}\label{thmau}
\citep[Theorem 2.1]{bartlett2005local}.
Let $F$ be a class of functions that map $\mathcal{X}$ into $[0,M]$.
Assume that there is some $\rho>0$ such that for every
$f\in F$, $\text{var}(f(X))\leq \rho$. Then, with probability at least $1-\delta$, we have
\begin{eqnarray*}
&&\sup_{f\in F}\left(\mathbb{E}f(X)-\frac{1}{n}\sum_{i=1}^{n}f(X_i)\right)\\
&&\leq \left(4\mathfrak{R}(F)+\sqrt{\frac{2\rho\log(1/\delta)}{n}}+\frac{4M}{3}\frac{\log(1/\delta)}{n}\right).
\end{eqnarray*}
\end{prop}

To prove the theorem, we also need to introduce the following auxiliary lemma.

Define
\begin{eqnarray*}
\mathcal{G}_r(Z)=\left\{\frac{r}{\max\{r,\mathbb{E}\ell(h,Z)\}}\ell(h,Z)|h\in B_{r}\right\}.
\end{eqnarray*}
It is evident that $\mathcal{G}_r\subseteq\{\alpha \ell\circ h|h\in B_{r}, \alpha\in [0,1]\}$. The following lemma is proven in \citep{bartlett2005local}.

\begin{lemm}\label{lemmaau}
Define
\[
V_r=\sup_{g\in
  \mathcal{G}_r}\left(\mathbb{E}g(Z)-\frac{1}{n}\sum_{i=1}^{n}g(Z_i)\right)~.
\]
For any $r>0$ and $a>1$, if $V_r\leq r/a$ then every $h\in B_{r}$ satisfies
\begin{eqnarray*}
\mathbb{E}\ell(h,Z)\leq \frac{a}{a-1}\frac{1}{n}\sum_{i=1}^{n}\ell(h,Z_i)+V_r.
\end{eqnarray*}
\end{lemm}

Now, we are ready to prove Theorem~\ref{thmtwo}.

\begin{proof}[Proof of Theorem~\ref{thmtwo}]
First, we introduce an inequality to build the connection between algorithmic stability and hypothesis complexity.
According to Lemma \ref{lemma1}, for any $a>1$ and $\delta>0$, with probability at least $1-\delta$, we have
\begin{flalign}\label{link}
&R(h_S)-\frac{a}{a-1}R_S(h_S)\leq \sup_{h\in B_r}(R(h)-\frac{a}{a-1}R_S(h))~.
\end{flalign}

Second, we are going to upper bound the term $\sup_{h\in B_r}(R(h)-\frac{a}{a-1}R_S(h))$ with high probability.
It is easy to check that for any $g\in\mathcal{G}_r$, $\mathbb{E}g(Z)\leq r$ and $g(Z)\in[0,M]$. Then
\begin{eqnarray*}
\text{var}(g(Z))\leq \mathbb{E}(g(Z))^2\leq M\mathbb{E}g(Z)\leq Mr.
\end{eqnarray*}
Applying Proposition \ref{thmau},
\begin{eqnarray*}
V_r\leq 4\mathfrak{R}(\mathcal{G}_r)+\sqrt{\frac{2Mr\log(1/\delta)}{n}}+\frac{4M}{3}\frac{\log(1/\delta)}{n}~.
\end{eqnarray*}

Let
\begin{eqnarray*}
4\mathfrak{R}(\mathcal{G}_r)+\sqrt{\frac{2Mr\log(1/\delta)}{n}}+\frac{4M}{3}\frac{\log(1/\delta)}{n}=\frac{r}{a}.
\end{eqnarray*}
We have
\begin{eqnarray*}
r\leq \frac{2Ma^2\log(1/\delta)}{n}+8a\mathfrak{R}(\mathcal{G}_r)+\frac{4}{3}\frac{2aM\log(1/\delta)}{n},
\end{eqnarray*}
which means that there exists an $r^*\leq \frac{2Ma^2\log(1/\delta)}{n}+8a\mathfrak{R}(\mathcal{G}_r)+\frac{4}{3}\frac{2aM\log(1/\delta)}{n}$ such that $V_{r^*}\leq r^*/a$ holds.
According to Lemma \ref{lemmaau}, for any $h\in B_{r}$, with probability at least $1-\delta$, we have
\begin{flalign*}
&\mathbb{E}\ell(h,Z)\leq \frac{a}{a-1}\frac{1}{n}\sum_{i=1}^{n}\ell(h,Z_i)+V_{r^*}\\
&\quad\leq \frac{a}{a-1}\frac{1}{n}\sum_{i=1}^{n}\ell(h,Z_i)+\frac{r^*}{a}
\\
&\quad\leq\frac{a}{a-1}\frac{1}{n}\sum_{i=1}^{n}\ell(h,Z_i)+\frac{2Ma\log(1/\delta)}{n}
\\
& \quad\quad +8\mathfrak{R}(\mathcal{G}_r)+\frac{4}{3}\frac{2M\log(1/\delta)}{n}.
\end{flalign*}

It is easy to verify that $\mathcal{G}_r\subseteq\{\alpha \ell\circ h|h\in B_{r}, \alpha\in [0,1]\}\subseteq\text{conv}B_{r}$.

By elementary properties of the Rademacher complexity (see, e.g.,
\citet{bartlett2003rademacher}), $H'\subseteq H$ implies $\mathfrak{R}(H')\leq \mathfrak{R}(H)$.
Then, with probability at least $1-\delta$, we have
\begin{eqnarray*}
&&\sup_{h\in B_{r}}\left(\mathbb{E}\ell(h,Z)-\frac{a}{a-1}\frac{1}{n}\sum_{i=1}^{n}\ell(h,Z_i)\right)\\
&&\leq\frac{2Ma\log(1/\delta)}{n}+8\mathfrak{R}(\ell\circ B_{r})+\frac{4}{3}\frac{2M\log(1/\delta)}{n}.
\end{eqnarray*}

The proof of Theorem \ref{thmtwo} is complete by combining the above inequality
with inequality (\ref{link}), the Talagrand Contraction Lemma, and Theorem \ref{thmone}.
\end{proof}

In the next section, we specialize the above results to some learning
algorithms by proving their uniform argument stability.

\section{Applications}\label{application}

Various learning algorithms have been proved to possess some kind of
stability. We refer the reader to
\cite{devroye1979distribution,lugosi1994posterior,bousquet2002stability,zhang2003leave,poggio2009sufficient,hardt2015train,liu2016algorithm}
for such examples, including
stochastic gradient descent methods, empirical risk minimization, and
non-parametric learning algorithms such as $k$-nearest neighbor rules and kernel regression.

\subsection{Empirical Risk Minimization}\label{applicaiton3}

Regularized empirical risk minimization has been known to be uniformly stable \citep{bousquet2002stability}.
Here we consider regularized empirical risk minimization (RERM)
algorithms of the following form.
The empirical risk (or the objective function) of RERM is formulated as
\begin{eqnarray*}\label{rerm}
R_{S,\lambda} (h)=\frac{1}{n}\sum_{i=1}^{n}\ell(h,Z_i)+\lambda N(h),
\end{eqnarray*}
where $N:h\in H\mapsto N(h)\in\mathbb{R}^+$ is a convex function. Its corresponding expected counterpart is defined as
\begin{eqnarray*}
R_\lambda (h)=\mathbb{E}\ell(h,Z)+\lambda N(h).
\end{eqnarray*}

\citet{bousquet2002stability} proved that $\ell_2$-regularized learning algorithms are $\beta(n)$-uniformly stable.
\citet{poggio2009sufficient} extended the result and studied a sufficient condition of the penalty term $N(h)$ to ensure uniform
$\beta(n)$-stability. As we now show, both of their proof methods are applicable to the analysis of uniform argument stability.

By exploiting their results, we show that stable RERM algorithms have strong generalization properties.
\begin{thm}\label{mainerm}
  Assume that $\mathfrak{B}$ is a separable Hilbert space.
  Suppose that the marginal distribution of the $X_i$
  is such that $\|X_i\|_*\le B$ with probability one, for some $B>0$
  and that the
  loss function is convex in $h$,
bounded by $M$ and $L$-Lipschitz.
Suppose that for some constants $C$ and $\xi>1$, the penalty function $N(h)$ satisfies
\begin{eqnarray}\label{stableregular}
&&N(h_S)+N(h_{S^i})-2N\left(\frac{h_S+h_{S^i}}{2}\right)\nonumber\\
&&\geq C\|h_S-h_{S^i}\|^\xi.
\end{eqnarray}
Then, for any $\delta>0$, and $a>1$, if $h_S$ is the output of RERM, with probability at least $1-2\delta$, we have
\begin{eqnarray*}
\lefteqn{
R(h_S)-\frac{a}{a-1}R_S(h_S) } \\
&\leq & 8LB\left(\frac{LB}{C\lambda n}\right)^{\frac{1}{\xi-1}}\sqrt{2\log(2/\delta)}\\
&+ &\frac{(6a+8)M\log(1/\delta)}{3n}.
\end{eqnarray*}
Specifically, when $N(h)=\|h\|^2$, (\ref{stableregular}) holds with $\xi=2$ and $C=\frac{1}{2}\left(\frac{M}{\lambda}\right)^{\frac{1}{2}}$.
\end{thm}

\noindent
\begin{proof}
The proof of Theorem \ref{mainerm} relies on the following result implied by \citet{poggio2009sufficient}.

\begin{prop}\label{poggio1}
Assume the conditions of Theorem \ref{mainerm}.
Then the RERM learning algorithm is $\beta(n)$-uniformly stable with
\[\beta(n)=\left(\frac{L^\xi B^\xi}{C\lambda n}\right)^{\frac{1}{\xi-1}},\]
and is $\alpha(n)$-uniformly argument stable with
\[\alpha(n)=\left(\frac{LB}{C\lambda n}\right)^{\frac{1}{\xi-1}}.\]
Specifically, when $N(h)=\|h\|_p^p$ and $1<p\leq2$, the condition \ref{stableregular} on the penalty function holds with $\xi=2$ and $C=\frac{1}{4}p(p-1)\left(\frac{M}{\lambda}\right)^{\frac{p-1}{p}}$, where $\|h\|_p^p=\sum_{r}|h_r|^p$ and $r$ is the index for the dimensionality.
\end{prop}

Theorem \ref{mainerm} follows by combining Theorem \ref{thmtwo} and
Proposition \ref{poggio1}.
\end{proof}

\subsection{Stochastic Gradient Descent}

Stochastic gradient descent (SGD) is one of the most widely used
optimization methods in machine learning.
\citet{hardt2015train} showed that parametric models trained by SGD
methods are uniformly stable.
Their results apply to both convex and non-convex learning problems
and provide insights for why SGD performs well in practice, in
particular, for deep learning algorithms.

Their results are based on the assumptions that the loss function
employed is both Lipschitz and smooth.
In order to avoid technicalities of defining derivatives in general
Hilbert spaces, in this section we assume that
$\mathfrak{B}=\mathcal{X}=\mathbb{R}^d$,
the $d$-dimensional Euclidean space.

\begin{defi}[Smooth]
A differentiable loss function $\ell(h,\cdot)$ is $s$-smooth if for all $h,h'\in H$, we have
\begin{eqnarray*}
\|\nabla_h\ell(h,\cdot)-\nabla_{h'}\ell(h',\cdot)\|\leq s\|h-h'\|,
\end{eqnarray*}
where $\nabla_xf(x)$ denotes the derivative of $f(x)$ with respect to $x$ and $s>0$.
\end{defi}

\begin{defi}[Strongly Convex]
A differentiable loss function $\ell(h,\cdot)$ is $\gamma$-strongly convex with respect to $\|\cdot\|$ if for all $h, h'\in H$, we have
\begin{eqnarray*}
&(\nabla_h \ell(h,\cdot)-\nabla_{h'} \ell(h',\cdot))^T(h-h')\geq \gamma\|h-h'\|^2,
\end{eqnarray*}
where $\gamma>0$.
\end{defi}

Theorem \ref{thmtwo} is applicable to the results of SGD when the general loss function $\ell(h,x)$  is $L$-Lipschitz, $s$-smooth, and $h$ is linear with respect to $x$. Note that our definition of $L$-Lipschitzness requires the loss function to be Lipschitz in the linear form $\left< h,x \right>$.

\begin{thm}\label{mainsgd}
Let the stochastic gradient update rule be given by $h_{t+1}=h_t-\alpha_t\nabla_h\ell(h_t,X_{i_t})$, where $\alpha_t>0$ is the learning rate
and $i_t$ is the index for choosing one example for the $t$-th update.
Let $h_T$ and $h_T^i$ denote the outputs of SGD run on sample $S$ and $S^i$, respectively.
Assume that $\|X\|_*\le B$ with probability one.
Suppose that the loss function is $L$-Lipschitz, $s$-smooth, and upper bounded
by $M$.
Let SGD is run with a monotonically non-increasing step size
$\alpha_t\leq c/t$, where $c$ is a universal constant, for $T$ steps.
 Then, for any $\delta>0$ and $a>1$, with probability at least
 $1-2\delta$, we have
\begin{eqnarray*}
\lefteqn{ R(h_T)-\frac{a}{a-1}R_S(h_T)  }\\
&\leq & 8BL\frac{1+1/s c}{n-1}(2cBL)^{\frac{1}{s c+1}}T^{\frac{s c}{s c+1}}\sqrt{2\log(2/\delta)}\\
&&+\frac{(6a+8)M\log(1/\delta)}{3n}.
\end{eqnarray*}
When the loss function $\ell$ is convex, $L$-admissible, $s$-smooth, and upper bounded by $M$,
suppose that SGD is run with step sizes $\alpha_t\leq 2/s$ for $T$
steps.
Then, for any $\delta>0$ and $a>1$, with probability at least $1-2\delta$,
\begin{eqnarray*}
\lefteqn{ R(h_T)-\frac{a}{a-1}R_S(h_T) }\\
&\leq & \frac{16B^2L^2}{n}\sum_{t=1}^{T}\alpha_t\sqrt{2\log(2/\delta)}\\
&&+\frac{(6a+8)M\log(1/\delta)}{3n}.
\end{eqnarray*}
Moreover, when the loss function $\ell$ is $\gamma$-strongly convex, $s$-smooth, and upper bounded by $M$, let the stochastic gradient update be given by $h_{t+1}=\Pi_\Omega(h_t-\alpha_t\nabla_h\ell(h_t,X_{i_t}))$, where $\Omega$ is a compact, convex set over which we wish to optimize and $\Pi_\Omega(\cdot)$ is a projection such that $\Pi_\Omega(f)=\arg\min_{h\in H}\|h-f\|$. If the loss function is further $L$-Lipschitz over the set $\Omega$ and the projected SGD is run with a constant step size $\alpha\leq1/s$ for $T$ steps. Then, for any $\delta>0$ and $a>1$, with probability at least $1-2\delta$, the projected SGD satisfies that
\begin{eqnarray*}
\lefteqn{ R(h_T)-\frac{a}{a-1}R_S(h_T)}\\
&\leq &\frac{16DB^2L^2}{\gamma n}\sqrt{2\log(2/\delta)}+\frac{(6a+8)M\log(1/\delta)}{3n}.
\end{eqnarray*}
\end{thm}

Note that any $\ell_2$ regularized convex loss function is strongly convex. \citet{bousquet2002stability} studied the
stability of batch methods. When the loss function is strongly convex, the stability of SGD is consistent with the result in
\citep{bousquet2002stability}.

While the above result only applies to $L$-Lipschitz loss functions as defined in Definition~\ref{def:lipschitz}, it does explain some
generalization properties of \emph{layer-wise} training of neural networks by stochastic gradient descent. In this once-common training
scheme (see, e.g., \citealp{bengio2007greedy}), one freezes the parameters of the network before/after a certain layer and performs SGD for
this single layer. It is easy to see that, as long as the activation function and the loss function (connected with the network) are Lipschitz-continuous in their inputs, the overall loss can easily satisfy the continuous conditions of Theorem~\ref{mainsgd}. This implies that the parameters in each layer may generalize well in a certain sense if SGD is employed with an early stop.


The proof of Theorem \ref{mainsgd} follows immediately from Theorem
\ref{thmtwo}, combined with
the following result implied by \citet{hardt2015train} (which is a collection of the results of Theorems 3.8, 3.9, and 3.12 therein).
\begin{prop}\label{sgdthm}
Let the stochastic gradient update be given by $h_{t+1}=h_t-\alpha_t\nabla_h\ell(h_t,Z_{i_t})$, where $\alpha_t>0$ is the learning rate
and $i_t$ is the index for choosing one example for the $t$-th update.
Let $h_T$ and $h_T^i$ denote the outputs of SGD running on sample $S$ and $S^i$ respectively.
When the loss function is $L$-Lipschitz and $s$-smooth, suppose that SGD is run with
monotonically non-increasing step size $\alpha_t\leq c/t$, where $c$ is a universal constant, for $T$ steps. Then,
\begin{eqnarray*}
\|h_T-h_T^i\| \leq\frac{1+1/s c}{n-1}(2cBL)^{\frac{1}{s c+1}}T^{\frac{s c}{s c+1}}.
\end{eqnarray*}
When the loss function $\ell$ is convex, $L$-Lipschitz, and $s$-smooth,
suppose that SGD is run with step sizes $\alpha_t\leq 2/s$ for $T$ steps. Then,
\begin{eqnarray*}
\|h_T-h_T^i\| \leq\frac{2BL}{n}\sum_{t=1}^{T}\alpha_t.
\end{eqnarray*}
Moreover, when the loss function $\ell$ is $\gamma$-strongly convex and $s$-smooth, let the stochastic gradient update be given by $h_{t+1}=\Pi_\Omega(h_t-\alpha_t\nabla_h\ell(h_t,Z_{i_t}))$, where $\Omega$ is a compact, convex set over which we wish to optimize and $\Pi_\Omega(\cdot)$ is a projection such that $\Pi_\Omega(f)=\arg\min_{h\in H}\|h-f\|$. If the loss function is $L$-Lipschitz over the set $\Omega$ and the projected SGD is run with constant step size $\alpha\leq1/s$ for $T$ steps. Then, the projected SGD satisfies algorithmic argument stability with
\begin{eqnarray*}
\|h_T-h_T^i\| \leq\frac{2BL}{\gamma n}.
\end{eqnarray*}
\end{prop}

\section{Conclusion}\label{conclusion}
We introduced the concepts of uniform argument stability and algorithmic hypothesis class, defined as the class of hypotheses that are
likely to be output by the learning algorithm. We proposed a general probabilistic framework to exploit
local estimates for the complexity of hypothesis class to obtain fast convergence rates for stable learning algorithms. Specifically,
we defined the algorithmic hypothesis class by observing that the
output of stable learning algorithms concentrates around
$\mathbb{E}h_S$. The Rademacher complexity defined on the algorithmic
hypothesis class then converges at the same rate as that of the
uniform argument stability in Hilbert space, which are of order
$O(1/n)$ for various learning algorithms, such as empirical risk
minimization and stochastic gradient descent.
We derived fast convergence rates of order $O(1/n)$ for their deformed generalization errors. Unlike previously published guarantees of
similar flavor, our bounds hold with high probability, rather than only in expectation.

Our study leaves some open problems and allows several possible extensions. First, the algorithmic hypothesis
class defined in this study depends mainly on the property of learning algorithms but little on the data distribution. It would be
interesting to investigate a way to define an algorithmic hypothesis class by considering both the algorithmic property and the data
distribution. Second, it would be interesting to explore if there are some algorithmic properties other than stability that could result in
a small algorithmic hypothesis class.

\section*{Acknowledgments}
Liu and Tao were partially supported by Australian Research Council Projects FT-130101457, DP-140102164, LP-150100671.
Lugosi was partially supported by the Spanish Ministry of Economy and Competitiveness, Grant MTM2015-67304-P, and FEDER, EU.
Neu was partially supported by the UPFellows Fellowship (Marie Curie COFUND program 600387).

\bibliography{reference}
\bibliographystyle{icml2017}

\end{document}